\documentclass{article}
 \PassOptionsToPackage{numbers}{natbib}
\usepackage[final]{nips_2017}

\usepackage[utf8]{inputenc} 
\usepackage[T1]{fontenc}    
\usepackage{hyperref}       
\usepackage{url}            
\usepackage{booktabs}       
\usepackage{amsfonts}       
\usepackage{nicefrac}       
\usepackage{microtype}      

\usepackage{amsmath}
\usepackage{amssymb}
\usepackage{xspace}
\usepackage{bm}
\usepackage{algorithm}

\newcommand{\F}{\mathcal{F}}

\newcommand{\real}{\mathbb{R}}

\newcommand{\II}[1]{\mathbb{I}_{\left\{#1\right\}}}
\newcommand{\PP}[1]{\mathbb{P}\left[#1\right]}
\newcommand{\EE}[1]{\mathbb{E}\left[#1\right]}

\newcommand{\PPc}[2]{\mathbb{P}\left[#1\left|#2\right.\right]}

\newcommand{\PPcc}[2]{\mathbb{P}\left[\left.#1\right|#2\right]}

\newcommand{\EEcc}[2]{\mathbb{E}\left[\left.#1\right|#2\right]}

\def\argmax{\mathop{\mbox{ arg\,max}}}

\newcommand{\ev}[1]{\left\{#1\right\}}
\newcommand{\pa}[1]{\left(#1\right)}

\newcommand{\wh}{\widehat}
\newcommand{\wt}{\widetilde}

\newcommand{\hmu}{\wh{\mu}}

\newcommand{\tmu}{\wt{\mu}}

\newcommand{\Et}{E^{\tmu}}
\newcommand{\Eh}{E^{\hmu}}
\newcommand{\Etti}{\Et_{t,i}}
\newcommand{\Ehti}{\Eh_{t,i}}

\usepackage{amsthm}

\newtheorem{theorem}{Theorem}
\newtheorem{proposition}{Proposition}
\newtheorem{lemma}{Lemma}
\newtheorem{corollary}{Corollary}

\usepackage{todonotes}
\definecolor{PalePurp}{rgb}{0.66,0.57,0.66}

\title{Boltzmann Exploration Done Right} 

\author{
  Nicol\`o Cesa-Bianchi \\
  Universit\`a degli Studi di Milano\\
  Milan, Italy\\
  \texttt{nicolo.cesa-bianchi@unimi.it}
  \And
  Claudio Gentile \\
  INRIA Lille -- Nord Europe\\
  Villeneuve d'Ascq, France \\
  \texttt{cla.gentile@gmail.com}
  \And
  G\'abor Lugosi \\
  ICREA \& Universitat Pompeu Fabra \\
  Barcelona, Spain \\
  \texttt{gabor.lugosi@gmail.com}
  \And
  Gergely Neu \\
  Universitat Pompeu Fabra \\
  Barcelona, Spain \\
  \texttt{gergely.neu@gmail.com}
 }

\begin{document}

\maketitle

\begin{abstract} 
Boltzmann exploration is a classic strategy for sequential decision-making under uncertainty, and is one of the most standard tools in 
Reinforcement Learning (RL). Despite its widespread use, there is virtually no theoretical understanding about the limitations or the 
actual benefits of this exploration scheme. Does it drive exploration in a meaningful way? Is it prone to misidentifying the optimal 
actions or spending too much time exploring the suboptimal ones? What is the right tuning for the learning rate? In this paper, we address 
several of these questions for the classic setup of stochastic multi-armed bandits. One of our main results is showing that the Boltzmann 
exploration strategy with any monotone learning-rate sequence will induce suboptimal behavior. As a remedy, we offer a simple non-monotone 
schedule that guarantees near-optimal performance, albeit only when given prior access to key problem parameters that are typically not 
available in practical situations (like the time horizon $T$ and the suboptimality gap $\Delta$). More importantly, we propose a novel 
variant that uses different learning rates for different arms, and achieves a distribution-dependent regret bound of order $\frac{K\log^2 
T}{\Delta}$ and a distribution-independent bound of order $\sqrt{KT}\log K$ without requiring such prior knowledge. To demonstrate 
the flexibility of our technique, we also propose a variant that guarantees the same performance bounds even if the rewards are 
heavy-tailed.
\end{abstract}

\section{Introduction}
Exponential weighting strategies are fundamental tools in a variety of areas, including Machine Learning, Optimization, Theoretical 
Computer Science, and Decision Theory \citep{ArHaKa12}. Within Reinforcement Learning \citep{sutton,Sze10}, exponential weighting schemes 
are broadly used for balancing exploration and exploitation, and are equivalently referred to as Boltzmann, Gibbs, or softmax exploration 
policies \citep{dyna,kaelbling1996reinforcement,SMSM99,PePre02}. In the most common version of Boltzmann exploration, the probability of 
choosing an arm is proportional to an exponential function of the empirical mean of the reward of that arm. Despite the popularity of this 
policy, very little is known about its theoretical performance, even in the simplest reinforcement learning setting of \emph{stochastic 
bandit problems}.

The variant of Boltzmann exploration we focus on in this paper is defined by
\begin{equation}\label{eq:Boltzmann}
 p_{t,i} \propto e^{\eta_t \hmu_{t,i}},
\end{equation}
where $p_{t,i}$ is the probability of choosing arm $i$ in round $t$, $\hmu_{t,i}$ is the empirical average of the rewards obtained from arm 
$i$ up until round $t$, and $\eta_t > 0$ is the \emph{learning rate}. This variant is broadly used in reinforcement learning 
\citep{sutton,Sze10,kaelbling1996reinforcement,VM05,kuleshov2014algorithms,OVW16}. In the multiarmed bandit literature, exponential-weights 
algorithms are also widespread, but they typically use \emph{importance-weighted} estimators for the rewards ---see, e.g., 
\citep{auer1995gambling,auer2002bandit} (for the nonstochastic setting), \citep{CBF98} (for the stochastic setting), and \citep{SS14} (for 
both stochastic and nonstochastic regimes). The theoretical behavior of these algorithms is generally well understood. For example, in the 
stochastic bandit setting \citet{SS14} show a regret bound of order $\frac{K\log^2 T}{\Delta}$, where $\Delta$ is the suboptimality gap 
(i.e., the smallest difference between the mean reward of the optimal arm and the mean reward of any other arm).

In this paper, we aim to achieve a better theoretical understanding of the \emph{basic} variant of the Boltzmann exploration policy that 
relies on the empirical mean rewards. We first show that any monotone learning-rate schedule will inevitably force the 
policy to either spend too much time drawing suboptimal arms or completely fail to identify the optimal arm. Then, we show that a specific 
non-monotone schedule of the learning rates can lead to regret bound of order $\frac{K\log T}{\Delta^2}$. However, the learning schedule has 
to rely on full knowledge of the gap $\Delta$ and the number of rounds $T$. Moreover, our negative result helps us to identify a crucial 
shortcoming of the Boltzmann exploration policy: it does not reason about the uncertainty of the empirical reward estimates. 
To alleviate this issue, we propose a variant that takes this uncertainty into account by using separate learning rates for each arm, where 
the learning rates account for the uncertainty of each reward estimate. We show that the resulting algorithm guarantees a 
distribution-dependent regret bound of order $\frac{K \log^2 T}{\Delta}$, and a distribution-independent bound of order $\sqrt{KT}\log K$.

Our algorithm and analysis is based on the so-called \emph{Gumbel--softmax} trick that connects the exponential-weights distribution with 
the maximum of independent random variables from the Gumbel distribution.

\section{The stochastic multi-armed bandit problem}
Consider the setting of stochastic multi-armed bandits: each arm $i\in[K]\stackrel{\mbox{\tiny{def}}}{=}\ev{1,2,\dots,K}$ yields a reward 
with distribution $\nu_i$, mean $\mu_i$, with the optimal mean reward being $\mu^* = \max_i \mu_i$. Without loss of generality, we will 
assume that the optimal arm is unique and has index~1. The gap of arm $i$ is defined as 
$\Delta_i = \mu^* - \mu_i$. 
We consider a repeated game between the learner and the 
environment, where in each round $t=1,2,\dots$, the following steps are repeated:
\begin{enumerate}
 \item The learner chooses an arm $I_t\in[K]$,
 \item the environment draws a reward $X_{t,I_t}\sim \nu_{I_t}$ independently of the past,
 \item the learner receives and observes the reward $X_{t,I_t}$.
\end{enumerate}
The performance of the learner is measured in terms of the \emph{pseudo-regret} defined as
\begin{equation}\label{eq:regret}
 R_T = \mu^*T - \sum_{t=1}^T \EE{X_{t,I_t}} = \mu^*T - \EE{\sum_{t=1}^T \mu_{I_t}} = \EE{\sum_{t=1}^T \Delta_{I_t}} = \sum_{i=1}^K \Delta_i 
\EE{N_{T,i}},
\end{equation}
where we defined $N_{t,i} = \sum_{s=1}^{t} \II{I_s = i}$, that is, the number of times that arm $i$ has been chosen until the end of round 
$t$. We aim at constructing algorithms that guarantee that the regret grows sublinearly. 

We will consider the above problem under various assumptions of the distribution of the rewards. For most of our results, we will assume 
that each $\nu_i$ is \emph{$\sigma$-subgaussian} with a known parameter $\sigma>0$, that is, that
\[
 \EE{e^{y \pa{X_{1,i} - \EE{X_{1,i}}}}} \le e^{\sigma^2 y^2/2}
\]
holds for all $y\in\real$ and $i\in[K]$. It is easy to see that any random variable bounded in an interval of length $B$ is 
$B^2/4$-subgaussian. Under this assumption, it is well known that any algorithm will suffer a regret of at least $\Omega\pa{\sum_{i>1} 
\frac{\sigma^2 \log T}{\Delta_i}}$, as shown in the classic paper of \citet{LR85}. There exist several algorithms guaranteeing matching 
upper bounds, even for finite horizons \citep{auer2002finite,cappe2013kullback,kaufmann12thompson}. We refer to the 
survey of \citet{bubeck12survey} for an exhaustive treatment of the topic.

\section{Boltzmann exploration done wrong}
We now formally describe the heuristic form of Boltzmann exploration that is commonly used in the reinforcement learning literature 
\citep{sutton,Sze10,kaelbling1996reinforcement}. This strategy works by maintaining the empirical estimates of each $\mu_i$ defined as
\begin{equation}\label{eq:mu_est}
 \hmu_{t,i} = \frac{\sum_{s=1}^t X_{s,i} \II{I_s = i}}{N_{t,i}}
\end{equation}
and computing the exponential-weights distribution (\ref{eq:Boltzmann})
for an appropriately tuned sequence of \emph{learning rate} parameters $\eta_t>0$ (which are often referred to as the \emph{inverse 
temperature}). As noted on several occasions in the literature, finding the right schedule for $\eta_t$ can be very difficult in practice 
\citep{kaelbling1996reinforcement,VM05}. Below, we quantify this difficulty by showing that natural learning-rate schedules may fail to 
achieve near-optimal regret guarantees. More precisely, they may draw suboptimal arms too much even after having estimated all the means 
correctly, 
or commit too early to a suboptimal arm and never recover afterwards. We partially circumvent this issue by proposing an admittedly 
artificial 
learning-rate schedule that actually guarantees near-optimal performance. However, a serious limitation of this schedule is that it relies 
on prior knowledge of problem parameters $\Delta$ and $T$ that are typically unknown at the beginning of the learning procedure. These 
observations lead us to the conclusion that the Boltzmann exploration policy as described by Equations~(\ref{eq:Boltzmann}) and 
\eqref{eq:mu_est} is no more effective for regret minimization than the simplest alternative of $\varepsilon$-greedy exploration 
\citep{sutton,auer2002finite}. 

Before we present our own technical results, we mention that \citet{SiJaLiSz00} propose a learning-rate schedule $\eta_t$ for Boltzmann 
exploration that simultaneously guarantees that all arms will be drawn infinitely often as $T$ goes to infinity, and that the policy 
becomes greedy in the limit. This property is proven by choosing a learning-rate schedule adaptively to ensure that in each round $t$, 
each arm gets drawn with probability at least $\frac 1t$, making it similar in spirit to $\varepsilon$-greedy exploration. While this 
strategy clearly leads to sublinear regret, it is easy to construct examples on which it suffers a regret of at least 
$\Omega\pa{T^{1-\alpha}}$ for any small $\alpha>0$. In this paper, we pursue a more ambitious goal: we aim to find out whether Boltzmann 
exploration can actually guarantee polylogarithmic regret. In the rest of this section, we present both negative and positive results 
concerning the standard variant of Boltzmann exploration, and then move on to providing an efficient generalization that achieves 
consistency in a more universal sense.

\subsection{Boltzmann exploration with monotone learning rates is suboptimal}
In this section, we study the most natural variant of Boltzmann exploration that uses a monotone learning-rate schedule. It is easy to see 
that in order to achieve sublinear regret, the learning rate $\eta_t$ needs to \emph{increase} with $t$ so that the suboptimal arms are 
drawn with less and less probability as time progresses. For the sake of clarity, we study the simplest possible setting with two arms with 
a gap of $\Delta$ between their means.
We first show that, asymptotically, the learning rate has to increase at least at 
a rate $\frac{\log t}{\Delta}$ even when the mean rewards are perfectly known. In other words, this is 
the minimal affordable learning rate.
\begin{proposition}
 Let us assume that $\hmu_{t,i} = \mu_i$ for all $t$ and both $i$. If $\eta_t = o\pa{\frac{\log(t\Delta^2)}{\Delta}}$, then the regret 
grows at least as fast as $R_T = \omega\pa{\frac{\log T}{\Delta}}$.
\end{proposition}
\label{prop:slow}
\begin{proof}
 Let us define $\eta^*_t = \frac{\log(t\Delta^2)}{\Delta}$ for all $t$. The probability of pulling the suboptimal arm can be 
asymptotically bounded as
 \begin{align*}
  \PP{I_t = 2} &= \frac{1}{1 + e^{\eta_t \Delta}} \ge \frac {e^{-\eta_t \Delta}}{2} = \omega\pa{\frac {e^{-\eta_t^* \Delta}}{2}} = 
\omega\pa{\frac{1}{\Delta^2 t}}.
 \end{align*}
 Summing up for all $t$, we get that the regret is at least
 \begin{align*}
  R_T = \Delta \sum_{t=1}^T \PP{I_t = 2} = \omega\pa{\sum_{t=1}^T \frac{1}{\Delta^2 t}} = \omega\pa{\frac{\log T}{\Delta}},
 \end{align*}
 thus proving the statement.
\end{proof}
This simple proposition thus implies an asymptotic lower bound on the schedule of learning rates $\eta_t$. In contrast, Theorem 
\ref{th:lowerboundgabor} below shows that all learning rate sequences that grow faster than $2\log t$ yield a linear regret, provided this 
schedule is adopted since the beginning of the game. This should be contrasted with Theorem \ref{th:upperbound}, which exhibits a schedule 
achieving logarithmic regret where $\eta_t$ grows faster than $2\log t$ only after the first $\tau$ rounds.
\begin{theorem}\label{th:lowerboundgabor}
There exists a 2-armed stochastic bandit problem with rewards bounded in $[0,1]$ where Boltzmann exploration using any learning rate 
sequence $\eta_t$ such that $\eta_t > 2\log t$ for all $t \geq 1$ has regret $R_T = \Omega(T)$. 
\end{theorem}
\begin{proof}
Consider the case where arm $2$ gives a reward deterministically equal to $\frac{1}{2}$  whereas the optimal arm $1$ has a Bernoulli 
distribution of parameter $p=\frac{1}{2}+\Delta$ for some $0 < \Delta < \frac{1}{2}$. Note that the regret of any algorithm satisfies $R_T 
\ge \Delta(T-t_0)\PP{\forall t>t_0 ,\ I_t = 2}$. Without loss of generality, assume that $\wh{\mu}_{1,1}=0$ and $\wh{\mu}_{1,2}=1/2$. Then 
for all $t$, independent of the algorithm, $\wh{\mu}_{t,2}=1/2$ and
\[
    p_{t,1}= \frac{e^{\eta_t\mathrm{Bin}(N_{t-1,1},p)}}{e^{\eta_t/2}+ e^{\eta_t\mathrm{Bin}(N_{t-1,1},p)}}
\quad \text{and} \quad
	p_{t,2}= \frac{e^{\eta_t/2}}{e^{\eta_t/2}+ e^{\eta_t\mathrm{Bin}(N_{t-1,1},p)}}~.
\]
For $t_0\ge 1$, Let $E_{t_0}$ be the event that $\mathrm{Bin}(N_{t_0,1},p)=0$, that is, up to time $t_0$, arm $1$ gives only zero reward 
whenever it is sampled. Then
\begin{align*}
	\PP{\forall t>t_0 \ I_t = 2}
&\ge
	\PP{E_{t_0}} \Big(1 - \PP{\exists t > t_0 \ I_t=1 \mid E_{t_0}}\Big)
\\&\ge
	\left(\frac{1}{2}-\Delta\right)^{t_0} \Big(1 - \PP{\exists t > t_0 \ I_t=1 \mid E_{t_0}}\Big).
\end{align*}
For $t>t_0$, let $A_{t,t_0}$ be the event that arm $1$ is sampled at time $t$ but not at any of the times $t_0+1,t_0+2,\ldots,t-1$.
Then, for any $t_0\ge 1$,
\begin{align*}
	\PP{\exists t > t_0 \ I_t=1 \mid E_{t_0}}
&=
	\PP{\exists t> t_0 \ A_{t,t_0} \mid E_{t_0}}
\le
	\sum_{t>t_0} \PP{A_{t,t_0} \mid E_{t_0}}
\\&=
	\sum_{t>t_0} \frac{1}{1+e^{\eta_t/2}} \prod_{s=t_0+1}^{t-1}\left(1- \frac{1}{1+e^{\eta_s/2}}\right)
\le
	\sum_{t>t_0} e^{-\eta_t/2}~.
\end{align*}
Therefore
\[
	R_T \ge \Delta(T-t_0) \left(\frac{1}{2}-\Delta\right)^{t_0} \left(1 - \sum_{t>t_0} e^{-\eta_t/2}\right)~.
\]
Assume $\eta_t \ge c\log t$ for some $c > 2$ and for all $t \ge t_0$. Then
\begin{align*}
	\sum_{t>t_0} e^{-\eta_t/2}
\le
	\sum_{t>t_0} t^{-\frac{c}{2}}
\le
	\int_{t_0}^{\infty} x^{-\frac{c}{2}}\,dx
=
	\left(\frac{c}{2}-1\right) t_0^{-(\frac{c}{2}-1)} \le \frac{1}{2}
\end{align*}
whenever $t_0 \ge (2a)^{\frac{1}{a}}$ where $a = \frac{c}{2}-1$. This implies $R_T = \Omega(T)$.
\end{proof}

\subsection{A learning-rate schedule with near-optimal guarantees}
\newcommand{\ve}{\varepsilon}
The above negative result is indeed heavily relying on the assumption that $\eta_t > 2\log t$ holds since the beginning. If we instead 
start off from a constant learning rate which we keep for a logarithmic number of rounds, then a logarithmic regret bound can be 
shown. Arguably, this results in a rather simplistic exploration scheme, which can be essentially seen as an \emph{explore-then-commit} 
strategy (e.g., \cite{gkl16}). Despite its simplicity, this strategy can be shown to achieve near-optimal performance if the parameters 
are tuned as a function the suboptimality gap $\Delta$ (although its regret scales at the suboptimal rate of $1/\Delta^2$ with this 
parameter). The following theorem (proved in Appendix~\ref{app:upperbound}) states this performance guarantee. 

\begin{theorem}\label{th:upperbound}
Assume the rewards of each arm are in $[0,1]$ and let
$
	\tau
=
	\frac{16eK\log T}{\Delta^2}
$.
Then the regret of Boltzmann exploration with learning rate $\eta_t = \II{t < \tau} + \frac{\log(t\Delta^2)}{\Delta}\II{t \ge \tau}$ 
satisfies
\[
	R_T \le \frac{16eK\log T}{\Delta^2} + \frac{9K}{\Delta^2}~.
\]
\end{theorem}

\section{Boltzmann exploration done right}
We now turn to give a variant of Boltzmann exploration that achieves near-optimal guarantees without prior knowledge of either $\Delta$ or 
$T$. Our approach is based on the observation that the distribution $p_{t,i}\propto \exp\pa{\eta_t \hmu_{t,i}}$ can be equivalently 
specified by the rule
$
 I_t = \argmax_j \ev{\eta_t \hmu_{t,j} + Z_{t,j}}
$,
where $Z_{t,j}$ is a standard Gumbel random variable\footnote{The cumulative density function of a standard Gumbel random variable is $F(x) 
= \exp(-e^{-x +\gamma})$ where $\gamma$ is the Euler-Mascheroni constant.} drawn independently for each arm $j$ (see, e.g., 
\citet{ALTS14} and the references therein). As we saw in the previous section, this scheme fails to guarantee consistency in general, as 
it does not capture the uncertainty of the reward estimates. We now propose a variant that takes this uncertainty into account by choosing 
different scaling factors for each perturbation. 
In particular, we will use the simple choice
$
 \beta_{t,i} = \sqrt{{C^2}\big/{N_{t,i}}}
$
with some constant $C>0$ that will be specified later. Our algorithm operates by independently drawing perturbations $Z_{t,i}$ from a 
standard Gumbel 
distribution for each arm $i$, then choosing action
\begin{equation}\label{eq:gumbel}
 I_{t+1} = \argmax_{i} \ev{\hmu_{t,i} + \beta_{t,i} Z_{t,i}}.
\end{equation}
We refer to this algorithm as \emph{Boltzmann--Gumbel exploration}, or, in short, BGE. 
Unfortunately, the probabilities $p_{t,i}$ no longer have a simple closed form, nevertheless the algorithm is very straightforward to 
implement. Our main positive result is showing the following 
performance guarantee about the algorithm.\footnote{We use the notation $\log_+(\cdot) = \max\{0,\cdot\}$.}
\begin{theorem}\label{thm:subgauss}
 Assume that the rewards of each arm are $\sigma^2$-subgaussian and let $c>0$ be arbitrary. Then, the regret of Boltzmann--Gumbel 
exploration satisfies
 \[
  R_T \le 
  \sum_{i=2}^K \frac{9C^2 \log^2_+\pa{T\Delta_i/c^2}}{\Delta_i}
  +
  \sum_{i=2}^K \frac{c^2 e^\gamma + 18C^2e^{\sigma^2/2C^2} \pa{1 + e^{-\gamma}}}{\Delta_i} + \sum_{i=2}^K \Delta_i.
 \]
 In particular, choosing $C = \sigma$ and $c = \sigma$ guarantees a regret bound of 
 \[
  R_T = O\pa{\sum_{i=2}^K \frac{\sigma^2 \log^2(T\Delta_i^2/\sigma^2)}{\Delta_i}}.
 \]
\end{theorem}
Notice that, unlike any other algorithm that we are aware of, Boltzmann--Gumbel exploration still continues to guarantee meaningful 
regret bounds even if the subgaussianity constant $\sigma$ is underestimated---although such misspecification is penalized exponentially in 
the true $\sigma^2$.
A downside of our bound is that it shows a suboptimal dependence on the number of rounds $T$: it grows asymptotically as 
$\sum_{i > 1}{\log^2(T\Delta_i^2)}\big/{\Delta_i}$, in contrast to the standard regret bounds for the UCB algorithm of 
\citet{auer2002finite} that grow as $\sum_{i > 1}({\log T})\big/{\Delta_i}$. However, our guarantee improves on the 
distribution-independent regret bounds of UCB that are of order $\sqrt{KT\log T}$. This is shown in the following corollary.
\begin{corollary}\label{cor:subgauss}
Assume that the rewards of each arm are $\sigma^2$-subgaussian. Then, the regret of Boltzmann--Gumbel exploration with $C = \sigma$ 
satisfies
$
  R_T \le 200 \sigma\sqrt{KT}\log K
$.
\end{corollary}
Notably, this bound shows optimal dependence on the number of rounds $T$, but is suboptimal in terms of the number of arms. To complement 
this upper bound, we also show that these bounds are tight in the sense that the $\log K$ factor cannot be removed.
\begin{theorem}\label{thm:wc}
For any $K$ and $T$ such that $\sqrt{K/T}\log K \le 1$, there exists a bandit problem with rewards bounded in $[0,1]$ where the regret of 
Boltzmann--Gumbel exploration with $C = 1$ is at least $R_T \ge \frac 12 \sqrt{KT}\log K$.
\end{theorem}
The proofs can be found in the Appendices~\ref{sec:cor} and~\ref{sec:wc}.
Note that more sophisticated policies are known to have better distribution-free bounds. The algorithm MOSS~\cite{audibert09minimax} 
achieves minimax-optimal $\sqrt{KT}$ distribution-free bounds, but distribution-dependent bounds of the form $(K/\Delta)\log(T\Delta^2)$ 
where $\Delta$ is the suboptimality gap. A variant of UCB using action elimination and due to \citet{auer10ucb2} has regret $\sum_{i > 
1}{\log(T\Delta_i^2)}\big/{\Delta_i}$ corresponding to a $\sqrt{KT(\log K)}$ distribution-free bound. The same bounds are achieved by the 
Gaussian Thompson sampling algorithm of \citet{AG13}, given that the rewards are subgaussian.

We finally provide a simple variant of our algorithm that allows to handle heavy-tailed rewards, intended here as reward distributions that 
are not subgaussian. We propose to use technique due to \citet{Cat12} based on the \emph{influence function}
\[
 \psi(x) = 
 \begin{cases}
 \log\pa{1+x + x^2/2}, &{\mbox{for $x\ge 0$}},\\
 -\log\pa{1-x + x^2/2}, &{\mbox{for $x\le 0$}}.
 \end{cases}
\]
Using this function, we define our estimates as
\[
 \hmu_{t,i} = \beta_{t,i} \sum_{s=1}^{t} \II{I_s = i} \psi\pa{\frac{X_{s,i}}{\beta_{t,i} N_{t,i}}}
\]

We prove the following result regarding Boltzmann--Gumbel exploration run with the above estimates.
\begin{theorem}\label{thm:heavy}
 Assume that the second moment of the rewards of each arm are bounded uniformly as $\EE{X_i^2} \le V$ and let $c>0$ be arbitrary. 
Then, the regret of Boltzmann--Gumbel exploration satisfies
 \[
  R_T \le 
  \sum_{i=2}^K \frac{9C^2 \log^2_+ \pa{T\Delta_i/c^2}}{\Delta_i}
  +
  \sum_{i=2}^K \frac{c^2 e^\gamma + 18C^2e^{V/2C^2} \pa{1 + e^{-\gamma}}}{\Delta_i} + \sum_{i=2}^K \Delta_i.
 \]
\end{theorem}
Notably, this bound coincides with that of Theorem~\ref{thm:subgauss}, except that $\sigma^2$ is replaced by $V$. Thus, by following the 
proof of Corollary~\ref{cor:subgauss}, we can show a distribution-independent regret bound of order $\sqrt{KT}\log K$.

\section{Analysis}
Let us now present the proofs of our main results concerning Boltzmann--Gumbel exploration, Theorems~\ref{thm:subgauss} 
and~\ref{thm:heavy}. Our analysis builds on several ideas 
from \citet{AG13}. We first provide generic tools that are independent of the reward estimator and then move on to providing specifics for 
both 
estimators.

We start with introducing some notation. We define
$
 \tmu_{t,i} = \hmu_{t,i} + \beta_{t,i} Z_{t,i}
$,
so that the algorithm can be simply written as $I_t = \argmax_i \tmu_{t,i}$. Let $\F_{t-1}$ be the sigma-algebra generated by the 
actions taken by the learner and the realized rewards up to the beginning of round $t$. Let us fix thresholds  
$x_{i},y_{i}$ satisfying $\mu_i \le x_i \le y_i \le \mu_1$ and define $q_{t,i} = \PPcc{\tmu_{t,1} > y_{i}}{\F_{t-1}}$. 
Furthermore, we define the events $E^{\hmu}_{t,i} = \ev{\hmu_{t,i} \le x_{i}}$ and $E^{\tmu}_{t,i} = \ev{\tmu_{t,i} \le y_{i}}$.
With this notation at hand, we can decompose the number of draws of any suboptimal $i$ as follows:
\begin{equation}\label{eq:decomp}
 \begin{split}
 \EE{N_{T,i}} =& \sum_{t=1}^T \PP{I_t = i, \Etti, \Ehti} + \sum_{t=1}^T \PP{I_t = i, \overline{\Etti}, \Ehti} + \sum_{t=1}^T \PP{I_t = i, 
\overline{\Ehti}}.
 \end{split}
\end{equation}
It remains to choose the thresholds $x_{i}$ and $y_{i}$ in a meaningful way: we pick
$x_{i} = \mu_i + \frac{\Delta_i}{3}$ and $y_{i} = \mu_{1} - \frac{\Delta_i}{3}$.
The rest of the proof is devoted to bounding each term in Eq.~\eqref{eq:decomp}. Intuitively, the individual terms capture the following 
events:
\begin{itemize}
 \item The first term counts the number of times that, even though the estimated mean reward of arm $i$ is well-concentrated and the 
additional perturbation $Z_{t.i}$ is not too large, arm $i$ was drawn instead of the optimal arm $1$. This happens when the optimal arm is 
poorly estimated or when the perturbation $Z_{t,1}$ is not large enough. Intuitively, this term measures the interaction between the 
perturbations $Z_{t,1}$ and the random fluctuations of the reward estimate $\hmu_{t,1}$ around its true mean, and will be small if the 
perturbations tend to be large enough and the tail of the reward estimates is light enough.
 \item The second term counts the number of times that the mean reward of arm $i$ is well-estimated, but it ends up being drawn due to a 
large perturbation. This term can be bounded independently of the properties of the mean estimator and is small when the tail 
of the perturbation distribution is not too heavy.
 \item The last term counts the number of times that the reward estimate of arm $i$ is poorly concentrated. This term is independent of the 
perturbations and only depends on the properties of the reward estimator.
\end{itemize}
As we will see, the first and the last terms can be bounded in terms of the \emph{moment generating function} of the reward estimates, 
which makes subgaussian reward estimators particularly easy to treat. We begin by the most standard part of our analysis: bounding the 
third term on the right-hand-side of~\eqref{eq:decomp} in terms of the moment-generating function.
\begin{lemma}\label{lem:SG_conc}
Let us fix any $i$ and define $\tau_k$ as the $k$'th time that arm $i$ was drawn. We have
 \[
\sum_{t=1}^T \PP{I_t = i, \overline{\Ehti}} \le 1 + \sum_{k=1}^{T-1}\EE{\exp\pa{\frac{\hmu_{\tau_k,i} - \mu_i}{\beta_{\tau_k,i}}}} \cdot 
e^{-\frac{\Delta_i \sqrt{k}}{3C}}.
 \]
\end{lemma}
Interestingly, our next key result shows that the first term can be bounded by a nearly identical expression:
\begin{lemma}\label{lem:invp}
 Let us define $\tau_k$ as the $k$'th time that arm $1$ was drawn. For any $i$, we have
 \[
  \sum_{t=1}^T \PP{I_t = i, \Etti, \Ehti} \le \sum_{k=0}^{T-1} \EE{\exp\pa{\frac{\mu_1 - \hmu_{\tau_k,1}}{\beta_{\tau_k,1}}}} 
e^{-\gamma-\frac{\Delta_i\sqrt{k}}{3C}}.
 \]
\end{lemma}
It remains to bound the second term in Equation~\eqref{eq:decomp}, which we do in the following lemma:
\begin{lemma}\label{lem:pert}
For any $i\neq 1$ and any constant $c > 0$, we have
 \[
\sum_{t=1}^T \PP{I_t = i, \overline{\Etti}, \Ehti} \le \frac{9C^2 \log^2_+\pa{T\Delta_i^2/c^2} + c^2 e^\gamma}{\Delta_i^2}.
 \]
\end{lemma}
The proofs of these three lemmas are included in the supplementary material.

\subsection{The proof of Theorem~\ref{thm:subgauss}}\label{sec:subgauss_proof}
For this section, we assume that the rewards are $\sigma$-subgaussian and that $\hmu_{t,i}$ is the empirical-mean estimator.
Building on the results of the previous section, observe that we are left with bounding the terms appearing in Lemmas~\ref{lem:SG_conc} 
and~\ref{lem:invp}. To this end, let us fix a $k$ and an $i$ and notice that by the subgaussianity assumption on the rewards, the 
empirical mean $\tmu_{\tau_k,i}$ is $\frac{\sigma}{\sqrt{k}}$-subgaussian (as $N_{\tau_k,i} = k$). In other words,
\[
 \EE{e^{\alpha \pa{\hmu_{\tau_k,i} - \mu_i}}} \le e^{\alpha^2\sigma^2/2k}
\]
holds for any $\alpha$. In particular, using this above formula for $\alpha = 1/\beta_{\tau_k,i} =\sqrt{\frac{k}{C^2}}$, we obtain
\[
 \EE{\exp\pa{\frac{\hmu_{\tau_k,i} - \mu_i}{\beta_{\tau_k,i}}}} \le e^{\sigma^2/2C^2}.
\]
Thus, the sum appearing in Lemma~\ref{lem:SG_conc} can be bounded as
\begin{align*}
 \sum_{k=1}^{T-1}\EE{\exp\pa{\frac{\hmu_{\tau_k,i} - \mu_i}{\beta_{\tau_k,i}}}} \cdot e^{-\frac{\Delta_i \sqrt{k}}{3C}}
 &\le e^{\sigma^2/2C^2} \sum_{k=1}^{T-1} e^{-\frac{\Delta_i \sqrt{k}}{3C}} \le \frac{18 C^2 e^{\sigma^2/2C^2}}{\Delta_i^2},
\end{align*}
where the last step follows from the fact\footnote{This can be easily seen by bounding the sum with an integral.} that
$ \sum_{k=0}^\infty e^{c\sqrt{k}} \le \frac{2}{c^2}$ holds for all $c>0$. The statement of Theorem~\ref{thm:subgauss} now follows from
applying the same argument to the bound of Lemma~\ref{lem:invp}, using Lemma~\ref{lem:pert}, and the standard expression for the regret in 
Equation~\eqref{eq:regret}. \qed

\subsection{The proof of Theorem~\ref{thm:heavy}}
We now drop the subgaussian assumption on the rewards and consider reward distributions that are possibly heavy-tailed, but have 
bounded variance.  The proof of Theorem~\ref{thm:heavy} trivially follows from the arguments in the previous subsection and using 
Proposition~2.1 of \citet{Cat12} (with $\theta = 0$) that guarantees the bound
\begin{equation}\label{eq:catoni}
 \EEcc{\exp\pa{\pm \frac{\mu_i - \hmu_{t,i}}{\beta_{t,i}}}}{N_{t,i} = n} \le \exp\pa{\frac{\EE{X_i^2}}{2C^2}}.
\end{equation}
\qed

\section{Experiments}
\begin{figure}
 \includegraphics[width=.5\textwidth]{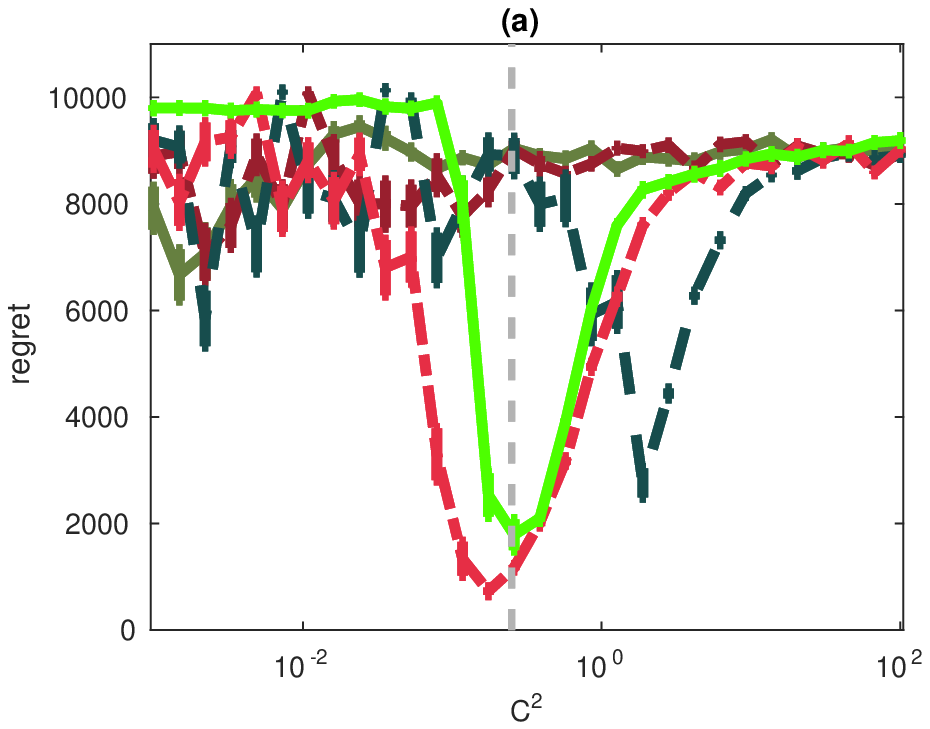}
 \includegraphics[width=.5\textwidth]{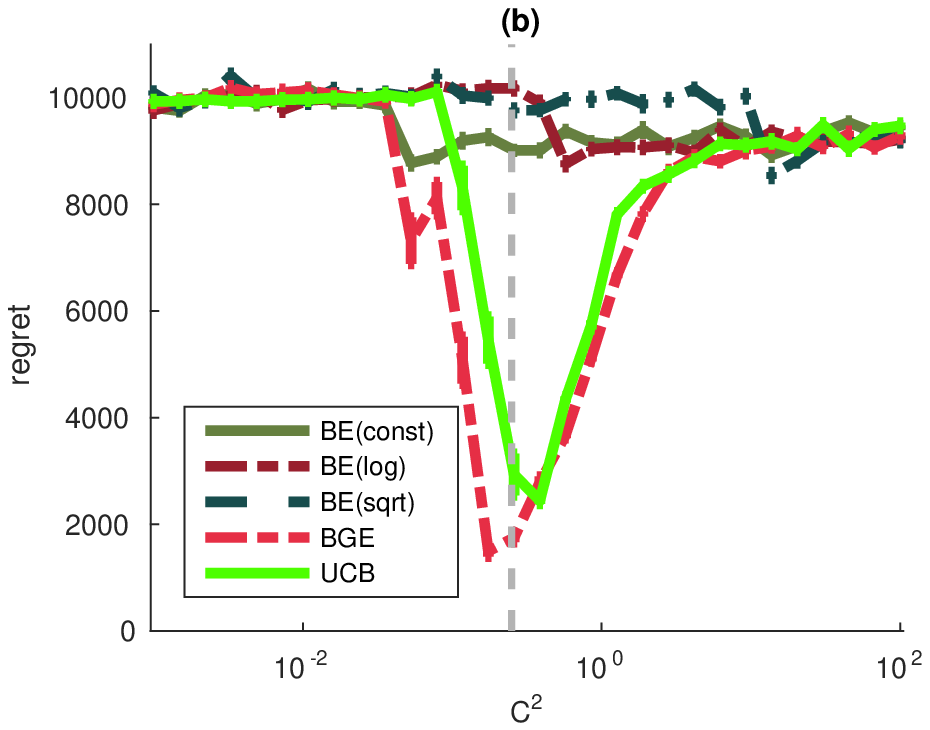}
 \caption{Empirical performance of Boltzmann exploration variants, Boltzmann--Gumbel exploration and UCB for (a) 
i.i.d.~initialization and (b) malicious initialization, as a function of $C^2$. The dotted vertical line corresponds to the 
choice $C^2 = 1/4$ suggested by Theorem~\ref{thm:subgauss}.}\label{fig:exp}
\end{figure}
This section concludes by illustrating our theoretical results through some experiments, highlighting the limitations of Boltzmann 
exploration and contrasting it with the performance of Boltzmann--Gumbel exploration. We consider a stochastic multi-armed bandit 
problem 
with $K=10$ arms each yielding Bernoulli rewards with mean $\mu_i = 1/2$ for all suboptimal arms $i>1$ and $\mu_1 = 1/2 + \Delta$ for the 
optimal arm. We set the horizon to $T = 10^6$ and the gap parameter to $\Delta = 0.01$. We compare three variants of Boltzmann 
exploration with inverse learning rate parameters
\begin{itemize}
 \item $\beta_t = C^2$ (\texttt{BE-const}),
 \item $\beta_t = C^2/\log t$ (\texttt{BE-log}), and
 \item $\beta_t = C^2/\sqrt{t}$ (\texttt{BE-sqrt})
\end{itemize}
for all $t$, and compare it with Boltzmann--Gumbel exploration (\texttt{BGE}), and UCB with exploration bonus $\sqrt{C^2 \log(t)/N_{t,i}}$. 

We study two different scenarios: (a) all rewards drawn i.i.d.~from the Bernoulli distributions with the means given above and (b) the 
first $T_0 = 5,\!000$ rewards set to $0$ for arm~$1$. The latter scenario simulates the situation described in the proof of 
Theorem~\ref{th:lowerboundgabor}, and in particular exposes the weakness of Boltzmann exploration with increasing learning rate 
parameters. The results shown on Figure~\ref{fig:exp}~(a) and~(b) show that while some variants of Boltzmann exploration may perform 
reasonably well when initial rewards take typical values and the parameters are chosen luckily, all standard versions fail to 
identify the optimal arm when the initial draws are not representative of the true mean (which happens with a small constant probability). 
On the other hand, UCB and Boltzmann--Gumbel exploration continue to perform well even under this unlikely event, as predicted by their 
respective theoretical guarantees. Notably, Boltzmann--Gumbel exploration performs 
comparably to UCB in this example (even slightly outperforming its competitor here), and performs notably well for the recommended 
parameter setting of $C^2 = \sigma^2 = 1/4$ (noting that Bernoulli random variables are $1/4$-subgaussian).

\paragraph{Acknowledgements} G\'abor Lugosi was supported by the Spanish Ministry of Economy and Competitiveness, Grant MTM2015-67304-P 
and FEDER, EU. Gergely Neu was supported by the UPFellows Fellowship (Marie Curie COFUND program n${^\circ}$ 600387).

\bibliographystyle{abbrvnat}
\bibliography{ngbib,allbib,confs}
\appendix
\section{Technical proofs}

\subsection{The proof of Theorem~\ref{th:upperbound}}\label{app:upperbound}
For any round $t$ and action $i$,
\begin{equation}\label{e:start}
	\frac{e^{-\eta_t}}{K} \le \frac{e^{\eta_t \hmu_{t-1,i}}}{\sum_{j=1}^K e^{\eta_t \hmu_{t-1,j}}} \le e^{\eta_t 
\big(\hmu_{t-1,i}-\hmu_{t-1,1}\big)}~.
\end{equation}
Now, for any $i > 1$, we can write
\begin{align*}
	\II{I_t = i}
&=
	\II{I_t = i,\, \hmu_{t-1,i} - \hmu_{t-1,1} < - \frac{\Delta_i}{2}} 
+       
	\II{I_t = i,\,\hmu_{t-1,i} - \hmu_{t-1,1} \ge - \frac{\Delta_i}{2}}
\\ &\le
	\II{I_t = i,\, \hmu_{t-1,i} - \hmu_{t-1,1} < - \frac{\Delta_i}{2}} 
+ 
	\II{\hmu_{t-1,1} \le \mu_1 - \frac{\Delta_i}{4}}
+
	\II{\hmu_{t-1,i} \ge \mu_i + \frac{\Delta_i}{4}}~.
\end{align*}
We take expectation of the three terms above and sum over $t = \tau+1,\ldots,T$. Because of~(\ref{e:start}), the first term is simply 
bounded as
\begin{align*}
	\sum_{t=\tau+1}^T \PP{I_t = i,\, \hmu_{t-1,i} - \hmu_{t-1,1} < - \frac{\Delta_i}{2}} 
\le 
        \sum_{t=\tau+1}^T e^{-\eta_t\Delta_i/2}
\le
	\sum_{t=\tau+1}^T \frac{1}{t\Delta^2}
\le
	\frac{\log (T+1)}{\Delta^2}~. 
\end{align*}
We control the second and third term in the same way. For the second term we have that
$
\II{\hmu_{t-1,1} \le \mu_1 - \frac{\Delta_i}{4}} \le \II{N_{t-1,1} \leq t_1} +  \II{\hmu_{t-1,1} \le \mu_1 - \frac{\Delta_i}{4},\, 
N_{t-1,1} 
> t_1}
$
holds for any fixed $t$ and for any $t_1 \le t-1$. Hence
\begin{align*}
	\sum_{t=\tau+1}^T \!\!\PP{\hmu_{t-1,1} \le \mu_1 - \frac{\Delta_i}{4}} 
\le\!\!
	\sum_{t=\tau+1}^T \!\!\PP{N_{t-1,1} \leq t_1}
+\!\!
	\sum_{t=\tau+1}^T \!\!\PP{\hmu_{t-1,1} \le \mu_1 - \frac{\Delta_i}{4},\, N_{t-1,1} > t_1}~.
\end{align*}
Now observe that, because of (\ref{e:start}) applied to the initial $\tau$ rounds, 
$\EE{N_{t-1,1}} \ge \frac{\tau}{eK}$ holds for all $t > \tau$. By setting $t_1 = \frac{1}{2}\EE{N_{t-1,1}} \ge \frac{\tau}{2eK}$, Chernoff 
bounds (in multiplicative form) give
$
\PP{N_{t-1,1} \leq t_1} \leq e^{-\frac{\tau}{8eK}}
$.
Standard Chernoff bounds, instead, give
\begin{align*}
	\PP{\hmu_{t-1,1} \le \mu_1 - \frac{\Delta_i}{4},\, N_{t-1,1} > t_1} 
\le
	\sum_{s=t_1+1}^{t-1} e^{-\frac{s\Delta^2}{8}}
\le
	\frac{8}{\Delta^2}e^{-\frac{t_1\Delta^2}{8}}
\le
	\frac{8}{\Delta^2}e^{-\frac{\tau\Delta^2}{16eK}}~.
\end{align*}
Therefore, for the second term we can write
\begin{align*}
	\sum_{t=\tau+1}^T \PP{\hmu_{t-1,1} \le \mu_1 - \frac{\Delta_i}{4}} 
\le
	T\,\left(e^{-\frac{\tau}{8eK}} + \frac{8}{\Delta^2}e^{-\frac{\tau\Delta^2}{16eK}}\right)
\le
	1 + \frac{8}{\Delta^2}~.
\end{align*}
The third term can be bounded exactly in the same way.
Putting together, we have thus obtained, for all actions $i > 1$,
\begin{align*}
	\sum_{i > 1} \EE{N_{T,i}} 
\le
	\tau + K + \frac{8K}{\Delta^2}
\le
	\frac{16eK(\log T)}{\Delta^2} + \frac{9K}{\Delta^2}~.
\end{align*}
This concludes the proof.
\qed

\subsection{The proof of Lemma~\ref{lem:SG_conc}}\label{app:SG_conc}
 Let $\tau_k$ denote the index of the round when arm $i$ is drawn for the $k$'th time. We let $\tau_0 = 0$ and $\tau_k = T$ for 
$k>N_{T,i}$. Then,
\begin{align*}
 \sum_{t=1}^T \PP{I_t = i, \overline{\Ehti}} &\le \EE{\sum_{k=0}^{T-1} \sum_{t=\tau_k+1}^{\tau_{k+1}} \II{I_t = i} \II{\overline{\Ehti}}}
 \\
 &= \EE{\sum_{k=0}^{T-1} \II{\overline{E^{\hmu}_{\tau_k,i}}} \sum_{t=\tau_k+1}^{\tau_{k+1}} \II{I_t = i}}
 \\
 &= \EE{\sum_{k=0}^{T-1} \II{\overline{E^{\hmu}_{\tau_k,i}}}}
 \\
 &\le 1 +  \sum_{k=1}^{T-1} \PP{\hmu_{\tau_k,i} \ge x_i}
 \\
 &\le 1 +  \sum_{k=1}^{T-1} \PP{\hmu_{\tau_k,i} - \mu_i \ge \frac{\Delta_i}{3}}.
\end{align*}
Now, using the fact that $N_{\tau_k,i} = k$, we bound the last term by exploiting the subgaussianity of the rewards through Markov's 
inequality:
\begin{align*}
 \PP{\hmu_{\tau_k,i} - \mu_i \ge \frac{\Delta_i}{3}} &= \PP{e^{\alpha \pa{\hmu_{\tau_k,i} - \mu_i}} \ge e^{\alpha 
\frac{\Delta_i}{3}}} \qquad\mbox{(for any $\alpha > 0$)}
 \\
 &\le \EE{e^{\alpha \pa{\hmu_{\tau_k,i} - \mu_i}}} \cdot e^{-\alpha \frac{\Delta_i}{3}} \qquad\mbox{(Markov's inequality)}
 \\
 &\le e^{\alpha^2\sigma^2/2k}\cdot e^{-\alpha \frac{\Delta_i}{3}} \;\;\;\;\;\qquad\qquad\mbox{(the subgaussian property)}
 \\
 &\le e^{\sigma^2/2C^2}\cdot e^{-\frac{\Delta_i\sqrt{k}}{3C}} \;\;\;\;\;\qquad\qquad\mbox{(choosing $\alpha = \sqrt{k/C^2}$)}
\end{align*}
Now, using the fact\footnote{This can be easily seen by bounding the sum with an integral.} that
$ \sum_{k=0}^\infty e^{c\sqrt{k}} \le \frac{2}{c^2}$
holds for all $c>0$, the proof is concluded. \qed

\subsection{The proof of Lemma~\ref{lem:invp}}
The proof of this lemma crucially builds on Lemma~1 of \citet{AG13}, which we state and prove below.
\begin{lemma}[cf.~Lemma~1 of \citet{AG13}]\label{lem:ivs1}
\begin{align*}
 \PPcc{I_t = i, E^{\hmu}_{t,i}, E^{\tmu}_{t,i}}{\F_{t-1}} \le  \frac{1-q_{t,i}}{q_{t,i}}\cdot \PPcc{I_t = 1, E^{\hmu}_{t,i}, 
E^{\tmu}_{t,i}}{\F_{t-1}}
\end{align*}
\end{lemma}
\begin{proof}
First, note that $E^{\hmu}_{t,i}\subseteq \F_{t-1}$. We only have to care about the case when $E^{\tmu}_{t,i}$ holds, otherwise both sides 
of the inequality are zero and the statement trivially holds. Thus, we only have to 
prove
\[
 \PPc{I_t = i}{\F_{t-1},E^{\tmu}_{t,i}} \le  \frac{1-q_{t,i}}{q_{t,i}}\cdot \PPc{I_t = 1}{\F_{t-1}, E^{\tmu}_{t,i}}.
\]
Now observe that $I_t = i$ under the event $E^{\tmu}_{t,i}$ implies $\tmu_{t,j} \le y_{i}$ for all $j$ (which follows from 
$\tmu_{t,j}\le\tmu_{t,i}\le y_i$). Thus, for any $i>1$, we have
\begin{align*}
\PPc{I_t = i}{\F_{t-1},E^{\tmu}_{t,i}} \le& \PPc{\forall j: \tmu_{t,j} \le y_{i}}{\F_{t-1},E^{\tmu}_{t,i}}
\\
=& \PPc{\tmu_{t,1} \le y_{i}}{\F_{t-1},E^{\tmu}_{t,i}}\cdot\PPc{\forall j>1: \tmu_{t,j} \le y_{i}}{\F_{t-1},E^{\tmu}_{t,i}}
\\
=& (1-q_{t,i})\cdot\PPc{\forall j>1: \tmu_{t,j} \le y_{i}}{\F_{t-1},E^{\tmu}_{t,i}},
\end{align*}
where the last equality holds because the event in question is independent of $E^{\tmu}_{t,i}$. Similarly,
\begin{align*}
\PPc{I_t = 1}{\F_{t-1},E^{\tmu}_{t,i}} \ge& \PPc{\forall j>1: \tmu_{t,1} > y_{i} \ge \tmu_{t,j} }{\F_{t-1},E^{\tmu}_{t,i}}
\\
=& \PPc{\tmu_{t,1} > y_{i}}{\F_{t-1},E^{\tmu}_{t,i}}\cdot\PPc{\forall j>1: \tmu_{t,j} \le y_{i}}{\F_{t-1},E^{\tmu}_{t,i}}
\\
=& q_{t,i}\cdot\PPc{\forall j>1: \tmu_{t,j} \le y_{i}}{\F_{t-1},E^{\tmu}_{t,i}}.
\end{align*}
Combining the above two inequalities and multiplying both sides with $\PPcc{E^{\tmu}_{t,i}}{\F_{t-1}}$ gives the result.
\end{proof}
We are now ready to prove Lemma~\ref{lem:invp}.
\begin{proof}[Proof of Lemma~\ref{lem:invp}]
Following straightforward calculations and using Lemma~\ref{lem:ivs1},
\begin{align*}
 \sum_{t=1}^T \PP{I_t = i, \Etti, \Ehti} &\le 
 \sum_{k=0}^{T-1} \EE{\frac{1-q_{\tau_k,i}}{q_{\tau_k,i}}}.
\end{align*}
Thus, it remains to bound the summands on the right-hand side.
To achieve this, we start with rewriting $q_{\tau_k,i}$ as
\begin{align*}
 q_{\tau_k,i} &= \PPcc{\tmu_{\tau_k,1} > y_{i}}{\F_{\tau_k-1}} = \PPcc{Z_{\tau_k,1} > \frac{\mu_1 - \hmu_{\tau_k,1} - 
\frac{\Delta_i}{3}}{\beta_{\tau_k,1}}}{\F_{\tau_k-1}}
 \\
 &= 1-\exp\pa{-\exp\pa{- \frac{\mu_1 - \hmu_{\tau_k,1} - \frac{\Delta_i}{3}}{\beta_{\tau_k,1}} +\gamma}},
\end{align*}
so that we have
\begin{align*}
 \frac{1-q_{\tau_k,i}}{q_{\tau_k,i}} &= \frac{\exp\pa{-\exp\pa{- \frac{\mu_1 - \hmu_{\tau_k,1} - \frac{\Delta_i}{3}}{\beta_{\tau_k,1}} 
+\gamma }}}{1-\exp\pa{-\exp\pa{- 
\frac{\mu_1 - \hmu_{\tau_k,1} - \frac{\Delta_i}{3}}{\beta_{\tau_k,1}} +\gamma}}}
\\
&\le \exp\pa{\frac{\mu_1 - \hmu_{\tau_k,1} - \frac{\Delta_i}{3}}{\beta_{\tau_k,1}} - \gamma} = \exp\pa{\frac{\mu_1 - 
\hmu_{t,1}}{\beta_{\tau_k,1}}} 
\cdot 
e^{- \gamma -\frac{\Delta_i}{3\beta_{\tau_k,1}}},
\end{align*}
where we used the elementary inequality $\frac{e^{-1/x}}{1-e^{-1/x}}\le x$ that holds for all $x\ge 0$. Taking expectations on both sides 
and using the definition of $\beta_{t,i}$ concludes the proof.
\end{proof}

\subsection{Proof of Lemma~\ref{lem:pert}}
Setting $L = \frac{9C^2\log^2 \pa{T \Delta_i^2 /c^2}}{\Delta_i^2}$, we begin with the bound
\[
 \sum_{t=1}^T \II{I_t = i, \overline{\Etti}, \Ehti} 
 \le L + \sum_{t=L}^T \II{\tmu_{t,i} > \mu_1 - \frac{\Delta_i}{3}, \hmu_{t,i} < \mu_i + \frac{\Delta_i}{3}, N_{t,i}>L}.
\]
For bounding the expectation of the second term above, observe that
\begin{align*}
 &\PPcc{\tmu_{t,i} > \mu_1 - \frac{\Delta_i}{3}, \hmu_{t,i} < \mu_i + \frac{\Delta_i}{3}, N_{t,i}>L}{\F_{t-1}}
 \le \PPcc{\tmu_{t,i}  > \hmu_{t,i} + \frac{\Delta_i}{3}, N_{t,i}>L}{\F_{t-1}}
 \\
 &\qquad\qquad\le \PPcc{\beta_{t,i} Z_{t,i} > \frac{\Delta_i}{3}, N_{t,i}>L}{\F_{t-1}}
 = \PPcc{Z_{t,i} > \frac{\Delta_i}{3\beta_{t,i}}, N_{t,i}>L}{\F_{t-1}}.
\end{align*}
By the distribution of the perturbations $Z_{t,i}$, we have
\begin{align*}
 \PPcc{Z_{t,i} \ge \frac{\Delta_i}{3\beta_{t,i} }}{\F_{t-1}} &= 1 - \exp\pa{-\exp\pa{-\frac{\Delta_i}{3\beta_{t,i}} + \gamma}}
 \\
 &\le \exp\pa{-\frac{\Delta_i}{3\beta_{t,i}} + \gamma} = \exp\pa{-\frac{\Delta_i\sqrt{N_{t,i}}}{3C} + \gamma},
\end{align*}
where we used the inequality $1 - e^{-x} \le x$ that holds for all $x$ and the definition of $\beta_{t,i}$. Noticing that $N_{t,i}$ is 
measurable in $\F_{t-1}$, we obtain the bound
\begin{align*}
 &\PPcc{Z_{t,i} > \frac{\Delta_i}{3\beta_{t,i}}, N_{t,i}>L}{\F_{t-1}} \le 
 \exp\pa{-\frac{\Delta_i\sqrt{N_{t,i}}}{3C} + \gamma}
\cdot \II{N_{t,i}>L},
\\
&\qquad\qquad \le \exp\pa{-\frac{\Delta_i\sqrt{L}}{3C} + \gamma} \cdot \II{N_{t,i}>L} \le \frac{c^2 e^\gamma}{T\Delta_i^2},
\end{align*}
where the last step follows from using the definition of $L$ and bounding the indicator by $1$. Summing up for all $t$ and taking 
expectations concludes the proof.
\qed

\subsection{The proof of Corollary~\ref{cor:subgauss}}\label{sec:cor}
Following the arguments in Section~\ref{sec:subgauss_proof}, we can show that the number of suboptimal draws can be bounded as
\[
 \EE{N_{T,i}} \le 1 + \sigma^2\frac{A + B \log^2 (T\Delta_i^2/\sigma^2)}{\Delta_i^2}
\]
for all arms $i$, with constants $A = e^\gamma+ 18 \sqrt{e} \pa{1 + e^{-\gamma}}$ and $B = 9$.
 We can obtain a distribution-independent bound by setting a threshold $\Delta>0$ and writing the regret as
\begin{align*}
 R_T &\le \sigma^2\sum_{i:\Delta_i>\Delta} \frac{A + B\log^2 (T\Delta_i^2/\sigma^2)}{\Delta_i} + \Delta T
 \\
 &\le \sigma^2 K \frac{A+ B\log^2 (T\Delta^2/\sigma^2)}{\Delta} + \Delta T \qquad\mbox{(since $\log^2(x^2)/x$ is monotone decreasing for 
$x\le 1$)}
 \\
 &\le \sigma\sqrt{TK} \frac{A+B \log^2 (K \log^2 K)}{\log K} + \sigma\sqrt{TK}\log K \qquad\mbox{(setting $\Delta = \sigma\sqrt{K/T}\log 
K$)}
 \\
 &\le \sigma\sqrt{TK} \frac{A+2 B \log^2 (K)}{\log K} + \sigma\sqrt{TK}\log K \;\;\qquad\qquad\mbox{(using $2\log\log(x) \le \log(x)$)}
 \\
 &\le \sigma\sqrt{TK} \log K \pa{2B + A/\log K}  + \sigma\sqrt{TK}\log K
 \\
 &\le \pa{2A+2B+1} \sigma\sqrt{TK} \log K,
\end{align*}
where we used $\log K \ge \frac 12$ that holds for $K\ge 2$. The proof is concluded by noting that $2A + 2B + 1 \approx 187.63 < 200$. \qed

\subsection{The proof of  Theorem~\ref{thm:wc}}\label{sec:wc}
The simple counterexample for the proof follows the construction of Section~3 of \citet{AG13}. Consider a problem with 
deterministic rewards for each arm: the optimal arm~1 always gives a reward of $\Delta = \sqrt{\frac{K}{T}}C_1$ 
and all the other arms give rewards of 0. Define $B_{t-1}$ as the event that $\sum_{i=2}^K N_{t,i} \le \frac{C_2\sqrt{KT}}{\Delta}$. Let 
us study two cases depending on the probability $\PP{A_{t-1}}$: If $\PP{A_{t-1}} \le \frac 12$, we have
\begin{align}\label{eq:notA}
 R_T \ge R_t &\ge \EEcc{\sum_i N_{t,i} \Delta}{\overline{A_{t-1}}}\cdot \frac 12.
 \ge \frac 12 C_2 \sqrt{KT}.
\end{align}

In what follows, we will study the other case when $\PP{A_{t-1}} \ge \frac 12$. We will show that, under this assumption, a suboptimal arm 
will be drawn in round $t$ with at least constant probability. In particular, we have
\begin{align*}
 \PP{I_t \neq 1} &= \PP{\exists i>1:\,\tmu_{t,1} < \tmu_{t,i}}
 \\
 &\ge \PP{\tmu_{t,1} < \mu_1,\, \exists i>1:\, \mu_1 < \tmu_{t,i}}
 \\
 &\ge \PPcc{\tmu_{t,1} < \mu_1,\, \exists i>1:\, \mu_1 < \tmu_{t,i}}{A_{t-1}} \PP{A_{t-1}}
 \\
 &\ge \EE{\PPcc{\tmu_{t,1} < \mu_1,\, \exists i>1:\, \mu_1 < \tmu_{t,i}}{\F_{t-1}, A_{t-1}}} \frac 12
 \\
 &= \EE{\PPcc{\tmu_{t,1} < \mu_1}{\F_{t-1},A_{t-1}}\cdot \PPcc{\exists i>1:\, \mu_1 < \tmu_{t,i}}{\F_{t-1}, A_{t-1}}} \frac 12
 \\
 &= \EE{\PPcc{Z_{t,1} < 0}{\F_{t-1},A_{t-1}} \cdot \PPcc{\exists i>1:\, \Delta < \beta_{t,i} Z_{t,i}}{\F_{t-1}, A_{t-1}}} \frac 12.
\end{align*}
To proceed, observe that $\PP{Z_{t,1} < 0} \ge 0.1$ and
\begin{align*}
&\PPcc{\exists i>1:\, \Delta < \beta_{t,i} Z_{t,i}}{\F_{t-1}, A_{t-1}} =
\PPcc{\exists i>1:\, \Delta \sqrt{N_{t,i}} < Z_{t,i}}{\F_{t-1}, A_{t-1}}
\\
&\qquad\qquad = 1 - \prod_{i>1} \exp\pa{-\exp\pa{- \Delta \sqrt{N_{t,i}} + \gamma}}
\\
&\qquad\qquad = 1 - \exp\pa{-\sum_{i>1} \exp\pa{- \Delta \sqrt{N_{t,i}} + \gamma}}
\\
&\qquad\qquad = 1 - \exp\pa{-\sum_{i>1} \frac{K-1}{K-1}\exp\pa{- \Delta \sqrt{N_{t,i}} + \gamma}}
\\
&\qquad\qquad \ge 1 - \exp\pa{-\pa{K-1}\exp\pa{- \Delta \sqrt{\sum_{i>1} \frac{N_{t,i}}{K-1}} + \gamma}} \qquad\mbox{(by Jensen's 
inequality)}
\\
&\qquad\qquad \ge 1 - \exp\pa{-\pa{K-1}\exp\pa{- \Delta \sqrt{\frac{C_2 \sqrt{KT}}{\Delta\pa{K-1}}} + \gamma}}
\\
&\qquad\qquad = 1 - \exp\pa{-\pa{K-1}\exp\pa{- \Delta \sqrt{\frac{C_2 T}{C_1 \pa{K-1}}} + \gamma}}
\\
&\qquad\qquad \ge 1 - \exp\pa{-\exp\pa{- C_1 \sqrt{\frac{C_2}{C_1}} + \log (K - 1) + \gamma}}.
\end{align*}
Setting $C_2 = C_1 = \log K$, we obtain that whenever $\PP{A_{t-1}} \ge \frac 12$, we have
\begin{align*}
 \PP{I_t \neq 1} &\ge 1 - \exp\pa{-\exp\pa{- \log K + \log (K - 1) + \gamma}}
 \\
 &\ge 1 - \exp\pa{-\exp\pa{\gamma}} \ge 0.83 > \frac 12.
\end{align*}
This implies that the regret of our algorithm is at least
\[
 \frac 12 T\Delta = \frac 12 \sqrt{TK}\log K.
\]
Together with the bound of Equation~\eqref{eq:notA} for the complementary case, this concludes the proof. \qed

\end{document}